%% file: main.tex
\documentclass{article} 
\usepackage{arxiv_v2}
\usepackage{times}  
\usepackage{helvet}  
\usepackage{courier}  
\usepackage[hyphens]{url}  
\usepackage{graphicx} 
\usepackage{booktabs}
\usepackage{caption} 
\DeclareCaptionStyle{ruled}{labelfont=normalfont,labelsep=colon,strut=off} 
\usepackage[colorlinks=true, citecolor=blue]{hyperref}

\usepackage{algorithm}
\usepackage[switch]{lineno}

\input{preamble}

\usepackage[citestyle=alphabetic,bibstyle=authortitle]{biblatex}
\newcommand{\citet}[1]{\citeauthor*{#1}~\cite{#1}}

\usepackage[symbol]{footmisc}

\title{Feature-based Individual Fairness in k-clustering}

\author{
    Debajyoti Kar\footnotemark[1]\\
    IIT Kharagpur, India\\
    \texttt{debajyoti.apeejay@gmail.com} \And
    Mert Kosan\footnotemark[1]\\
    University of California, Santa Barbara\\
    \texttt{mertkosan@ucsb.edu} \And
    Debmalya Mandal\\
    MPI-SWS, Germany\\
    \texttt{dmandal@mpi-sws.org} \And
    Sourav Medya\\
    University of Illinois, Chicagao\\
    \texttt{medya@uic.edu} \And 
    Arlei Silva\\
    Rice University, USA\\
    \texttt{arlei@rice.edu} \And
    Palash Dey\\
    IIT Kharagpur, India\\
    \texttt{palash.dey@cse.iitkgp.ac.in} \And
    Swagato Sanyal\\
     IIT Kharagpur, India\\
     \texttt{sanyalswagato@gmail.com}
}

\begin{document}

\maketitle

\footnote{Both authors contributed equally to this research.}

\begin{abstract}
Ensuring fairness in machine learning algorithms is a challenging and essential task. We consider the problem of clustering a set of points while satisfying fairness constraints. While there have been several attempts to capture group fairness in the $k$-clustering problem, fairness at an individual level is relatively less explored. We introduce a new notion of individual fairness in $k$-clustering based on features not necessarily used for clustering. We show that this problem is NP-hard and does not admit a constant factor approximation. Therefore, we design a randomized algorithm that guarantees approximation both in terms of minimizing the clustering distance objective and individual fairness under natural restrictions on the distance metric and fairness constraints. Finally, our experimental results against six competing baselines validate that our algorithm produces individually fairer clusters than the fairest baseline by 12.5\% on average while also being less costly in terms of the clustering objective than the best baseline by 34.5\% on average.
\end{abstract}

\input{sec_intro}

\input{sec_prelim}
\input{sec_results}

\input{sec_experiments}
\input{sec_conclusion}

\printbibliography

\appendix
\input{appendix.tex}

\end{document}

%% file: preamble.tex
\usepackage{xspace}
\usepackage{tikz}
\usepackage{color}
\definecolor{antiquebrass}{rgb}{0.8, 0.58, 0.46}
\usepackage{algpseudocode,algorithm}

\algrenewcommand\algorithmicrequire{\textbf{Input:}}
\algrenewcommand\algorithmicensure{\textbf{Output:}}
\algnewcommand{\Initialize}[1]{%
  \State \textbf{Initialize:}
  \Statex \hspace*{\algorithmicindent}\parbox[t]{.8\linewidth}{\raggedright #1}
}
\algdef{SE}[FUNCTION]{Function}{EndFunction}%
   [2]{\algorithmicfunction\ \textproc{#1}\ifthenelse{\equal{#2}{}}{}{(#2)}}%
   {\algorithmicend\ \algorithmicfunction}%

\renewcommand{\Return}{\State \textbf{return} }
\usepackage{epsfig}
\usepackage{subfig}
\usepackage{mathrsfs,amsfonts,amsmath,amsthm,mdwlist,bbm, amssymb}

\usepackage{makecell}
\usepackage{multirow}

\allowdisplaybreaks

\renewcommand{\ge}{\geqslant}
\renewcommand{\le}{\leqslant}

\usepackage{enumerate}

\usepackage[shortlabels]{enumitem}

\newtheorem{definition}{Definition}

 \setlist{nolistsep,leftmargin=*}

\usepackage{nicefrac}

\newcommand{\eps}{\ensuremath{\varepsilon}\xspace}
\renewcommand{\epsilon}{\eps}
\let\mydelta\delta
\renewcommand{\delta}{\ensuremath{\mydelta}\xspace}
\let\myalpha\alpha
\renewcommand{\alpha}{\ensuremath{\myalpha}\xspace}
\let\mystar\star
\renewcommand{\star}{\ensuremath{\mystar}}

\newcommand{\VC}{{\sc Vanilla Clustering}\xspace}
\newcommand{\IFC}{{\sc Individually Fair Clustering}\xspace}

\newcommand{\IFR}{{\sc Individually Fair Assignment}\xspace}
\newcommand{\IFA}{{\sc Individually Fair Assignment}\xspace}
\newcommand{\SPR}{{\sc Satisfactory-Partition}\xspace}

\newcommand{\TFC}{{\sc Trivially Fair Clustering}\xspace}

\newcommand{\el}{\ensuremath{\ell}\xspace}

\newcommand{\NP}{\ensuremath{\mathsf{NP}}\xspace}

\newcommand{\Pb}{\ensuremath{\mathsf{P}}\xspace}

\newcommand{\NPC}{\ensuremath{\mathsf{NP}}-complete\xspace}

\renewcommand{\AA}{\ensuremath{\mathcal A}\xspace}
\newcommand{\BB}{\ensuremath{\mathcal B}\xspace}

\newcommand{\EE}{\ensuremath{\mathcal E}\xspace}

\newcommand{\GG}{\ensuremath{\mathcal G}\xspace}

\newcommand{\VV}{\ensuremath{\mathcal V}\xspace}

\newcommand{\abs}[1]{\left| #1 \right|}
\newcommand{\set}[1]{\left\{ #1 \right\}}
\newcommand{\cost}{\mathrm{Cost}}

\newtheorem{theorem}{\bf Theorem}

\newtheorem{claim}{\bf Claim}

\usepackage{cleveref}

\crefname{theorem}{Theorem}{Theorems}
\crefname{observation}{Observation}{Observations}
\crefname{lemma}{Lemma}{Lemmas}
\crefname{corollary}{Corollary}{Corollaries}
\crefname{proposition}{Proposition}{Propositions}
\crefname{example}{Example}{Examples}
\crefname{claim}{Claim}{Claims}
\crefname{table}{Table}{Tables}
\crefname{equation}{Inequality}{Inequalities}
\crefname{reductionrule}{Reduction rule}{Reduction rules}
\crefname{section}{Section}{Sections}

%% file: sec_intro.tex
\section{Introduction}
Machine learning systems are increasingly being used in various societal decision-making, including predicting recidivism  \cite{ALMK16,Chouldechova2017}, deciding interest rates~\cite{FGRW20}, and even allocating healthcare resources~\cite{OPVM19}. However, beginning with the report on bias in recidivism risk prediction~\cite{ALMK16}, it has been known that such systems are often biased against certain groups of people. In recent years, various methods and definitions have been proposed for ensuring fairness in supervised learning settings, with efforts ranging from debiasing datasets~\cite{FFMS+15} to explicitly encoding the fairness constraints during the training of a classifier~\cite{ABDL+18}. 

This paper focuses on fairness in unsupervised learning, particularly clustering. 
There are two major reasons why clustering should be fair with respect to different subgroups. First, clustering is often a pre-processing step for generating new data representations for downstream tasks. 
Since we want the downstream decisions to be fair, the clustering step needs to be unbiased ~\cite{AFSV19}. 
Second, clustering is also used in various resource allocation problems, e.g. in \emph{facility location} \cite{JKL19}. 
Since it is desirable that no group is disproportionately affected by such decisions, there has been an increasing interest in designing clustering algorithms that are fair with respect to different subgroups \cite{CKLV17,BIOS+19,bera2019fair,AEKM19}.
Such group-fair clustering algorithms ensure that each protected group has an approximately equal presence in each cluster. 

Compared to group fairness, individually fair clustering has received less attention. Individually fair clustering is motivated by the \emph{facility location problem} where the goal is to open $k$ facilities while minimizing the total transportation cost between individuals and their nearest facility. If we choose $k$ facilities (or centers) uniformly at random, then each point $x$ could expect one of its nearest $n/k$ neighbors to be one of such facilities. This led a few studies \cite{JKL19,MV20,chakrabarty2021better} to consider the following notion of individual fairness. For a point $x$, let $r(x)$ be the radius such that the ball of radius $r(x)$ centered at $x$ has at least $n/k$ points. An individually fair clustering guarantees that, for every $x$, a cluster center is chosen from the $r(x)$-neighborhood of $x$.


Although individually fair clustering ~\cite{JKL19} provides guarantees for each point, 
it does not exactly reflect the original premise of individual fairness suggested by \cite{DHPR+12}, which requires that similar individuals should receive similar decisions. In the context of clustering, this means that two points $x$ and $y$ that are similar (in terms of features) should be clustered similarly. However, the definition proposed by \cite{JKL19} does not provide such a guarantee, as points similar to a point $x$ could be different from the points within a radius of $r(x)$ from $x$.

\paragraph{Proposed Definition of Individual Fairness.} In order to address the drawback above, we propose a new notion of individual fairness in clustering. 
First, motivated by the original definition of individual fairness in supervised learning \cite{DHPR+12}, we introduce a feature-based notion of individual fairness. We say that two individuals are similar if their features match significantly (parameterized by $\gamma$ in Definition \ref{def:gamma_sim}). Now, for each individual $v$, let $C(v)$ denote the cluster $v$ is assigned to. Then our feature-based individually fair clustering requires that $C(v)$ also contains at least $m_v$ individuals that are similar to $v$. The variable $m_v$ is a parameter to encode the degree of fairness. More specifically, it encodes the amount of similarity an individual seeks inside their own cluster. This guarantees that a point $v$ is not isolated in its own cluster but that the cluster has a desired representation (or participation) from points similar to it. Note that, the features that are used to compute similarity for individual fairness might not necessarily be used for clustering. In fact, these two sets of features might be disjoint. 


Our notion of individual fairness guarantees that similar individuals (in terms of possibly sensitive features) often share similar clusters. Consider the following motivating example. Ad networks collect user behavior data (e.g., browsing history, location) as well as possibly sensitive attributes (e.g., race and gender) to cluster users into several categories \cite{yan2009much}. These categories are directly used for targeted recommendations, including jobs and healthcare.  In this context, the cost function of the clustering algorithm should be based on user behavior while an individual’s notion of fairness should be based on sensitive attributes such as race and gender. In this case, similarity in terms of sensitive features can be seen as a relaxation of a protected group membership.



\paragraph{Contributions.} Our main contributions are as follows:
\begin{itemize}
    \item \textbf{Novel Formulation: }We propose a new definition of individual fairness in clustering based on how individuals are similar in terms of their features. Our definition guarantees that each individual has a desired level of representation of similar individuals in their own cluster.
    \item \textbf{Problem Characterization: }We show that minimizing the clustering cost subject to the new notion of individual fairness is NP-hard, and also cannot be approximated within a factor $\delta$ for any $\delta>0$.
    \item \textbf{Algorithm: }We design a randomized algorithm providing an additive approximation cost while guaranteeing fairness within a multiplicative factor with high probability.  
    \item \textbf{Experiments: } Our experiments on several standard datasets show that our approach produces by $34.5\%$ less cost on average in clustering than the best-competing method 
    while ensuring individual fairness for more than $95.5\%$ points on average.
\end{itemize}

\paragraph{Related Work.} \cite{CKLV17} first introduced the problem of fair clustering with disparate impact constraints and their goal was to ensure that all the protected groups have approximately equal representation in every cluster. Several works \cite{BDMS+19,RS18} studied different generalizations of the fair clustering problem. Furthermore, several papers 
\cite{BIOS+19,SSS18,HJV19} 
 proposed procedures to scale fair clustering to a large number of points. Although we consider individual fairness, our work is related to \cite{bera2019fair}, which shows that a $\rho$-approximation to the vanilla clustering problem can be converted to a $(\rho+2)$-approximate solution to fair clustering with bounded (and often negligible) violation of fairness constraints.

Our paper is focused on individual fairness, which was first defined by ~\cite{DHPR+12} in the context classification, and requires similar individuals to be treated similarly. For clustering problems, such a notion of individual fairness was first defined by \cite{JKL19}. 
They studied individual fairness in terms of the guarantee a randomly chosen set of $k$ points must satisfy. Informally, an individually fair clustering guarantees that for each point $x$, a cluster center is chosen from a certain neighborhood of $x$. \cite{MV20} designed a bicriteria approximation algorithm for individually fair $k$-means and $k$-median problems. Their algorithms guarantee that not only the fairness constraints are approximately satisfied, but also the objective is approximately maximized. Later, \cite{chakrabarty2021better} proposed an algorithm that has theoretical fairness guarantees comparable with \cite{MV20}, and empirically, obtains noticeably fairer solutions. Recently, \cite{VY21} designed improved bicriteria algorithms for general $\ell_p$-norm costs. Another recent study \cite{chhaya2022coresets} defined a coreset for individually fair clustering problem using the generalized fair radius notion, and \cite{chakrabarti2022new} used per-point fairness and aggregate fairness constraints for the k-center problem. Later they incorporated the price of fairness notion to combine these two constraints into one algorithm.

The definition of individual fairness in \cite{JKL19} was mainly motivated by fairness in the facility location problem. Recently, \cite{KAM20} considered a different notion of individual fairness in clustering, where the goal is to ensure that each point, is closer to the points in its own cluster than the points in any other cluster.  
Our proposed definition can be seen as a way to capture these two notions, as we consider feature-based similarity, as well as guaranteed representation for each point.

Here, we focus on the $\ell_p$-norm cost for clustering, which is just the sum of $\ell_p$-distances of each point from its corresponding cluster center. \cite{chakrabarti2022new} also considers $\ell_p$-norm objectives in their individual fairness formulation. However, several papers did consider other objectives in the context of group-fair clustering ~\cite{AEKM19,GSV21,KAM19}. Finally, our focus is on fair clustering algorithms, and there is extensive literature on fair algorithms for unsupervised~\cite{STMS+18,KSAM19} and supervised learning more broadly~\cite{ABDL+18,DOBS+18,CHKV19} . The coverage of these algorithms is out of the scope of the paper, and we refer the interested reader to the following excellent surveys:~\cite{CR18,SBFV+19} and \cite{MMSL+21}.

%% file: sec_prelim.tex
\section{Preliminaries}
We first introduce some necessary notations. Let $V$ be a set of $n$ points $V=\{1,2,\ldots, n\}$. We denote $\{S_1,S_2,\ldots,S_q\}$ as a set of $q$ features, where $S_i$ is the set of values for the $i$-th feature. We denote the tuple of $q$ features of the point $i$ by $X_i=(X_i^t)_{t\in[q]}$. We write $C=(C_i)_{i \in [k]}$ to denote a clustering (i.e. partition) of the set $V$ and $(c_i)_{i \in [k]}$ to denote the corresponding cluster centers. Given a clustering $C$ and a point $v$, let $\phi(v,C)$ be the cluster center assigned to the point $v$. When the clustering $C$ is clear from the context, we use $\phi(v)$ to denote the cluster center assigned to the vertex $v$. We are also given a distance function $d: V \times V \rightarrow \mathbb{R}$ that measures the distance between any pair of points. Given a clustering $C$, we can measure its cost through the distance metric $d$. In particular, we will be interested in measuring the sum of the $p$-th powers of distances from each point to its cluster center for $p \in \mathbb{N} \cup \set{0}$:
\begin{equation}\label{eq:p-cost}
\cost(C) = \sum_{v \in V} d(v, \phi(C,v))^p.
\end{equation}
We assume that the distance function $d$ depends on some features of the points but don't assume any relationship between those features and the ones used for fairness.

\subsection{Similarity}
In order to define the feature-based notion of individual fairness, we first define a similarity measure based on the features. To the best of our knowledge, all the existing notions of individual fairness in clustering only depend on the distance-based neighborhood of each point. In contrast, our definition of individual fairness builds upon the features of individual points that are not necessarily used for clustering. In order to define the feature-based notion of fairness, we first define a similarity measure based on the features.

In real-world settings, the fairness features can be both continuous and discrete. To handle both cases, we convert the discrete variables (features) into one-hot encoding vectors. The continuous variables are also normalized to be within the range $[0,1]$. After these conversions, we can now define the distance (or similarity) between two vectors. 
We convert the distance to similarity with the following where $d^\prime$ is a distance function on feature space:
\begin{equation}
    s(X_i, X_j) = e^{-d^\prime(X_i, X_j)}
\end{equation}
where $s$ is the similarity between $X_i$ and $X_j$ and $d^\prime$ is a distance function (e.g., Euclidean). This operation guarantees that $s$ will always generate a value between 0 and 1. We say that $X_i$ and $X_j$ are $gamma$ similar if $s > \gamma$.
\begin{definition}[\textbf{$\gamma$-similarity}]
\label{def:gamma_sim}
   For a parameter $\gamma\in[0,1]$, we say two points $i,j\in V, i\ne j$ are $\gamma$-similar if $s(X_i,X_j)>\gamma$ where $s(X_i, X_j) = e^{-d^\prime(X_i, X_j)}$. We assume that a point is not $\gamma$-matched with itself.  
\end{definition}

Our definition of similarity is flexible enough to support diverse applications. For any point $v$, we use $\Gamma(v)$ to denote the set of points in $V$ that are $\gamma$-similar to $v$. Next, we introduce our definition of individually fair clustering.


%

\subsection{Individual Fairness}
\begin{definition}[Individual Fairness in Clustering]
    Given a set $V$ of $n$ points along with a $q$-length feature vector $X_v=(x_v^1,\ldots,x_v^q)$ for every point $v\in V$, a similarity parameter $\gamma\in[0,1]$, an integer tuple $(m_v)_{v\in V}$, and an integer $k$, we say that a clustering $(C_i)_{i \in [\ell]}$ ($\ell \le k$) is 
    $(m_v)_{v\in V}$-individually fair if it satisfies the following constraint for every point $v \in V$:
    \begin{equation}\label{eq:p-ifc-constraint}
       \abs{\set{u : u \in \Gamma(v) \text{ and } \phi(u) = \phi(v)} } \ge m_v
    \end{equation}
      
\end{definition}
The fairness constraint~\eqref{eq:p-ifc-constraint} says that at least $m_v$ points that are $\gamma$-similar to point $v$ must belong to the cluster of $v$. Our goal is to cluster $V$ into $\el\; (\le k)$ clusters,  $(C_i)_{i\in[\el]}$, with corresponding centers (or facilities\footnote{We use cluster center and facility interchangeably.}) $(c_i)_{i\in[\el]}$, such that clusters are individually fair for every point and minimize the clustering cost (e.g. sum of the $p$-th powers of distances from cluster centers for some $p \in \mathbb{N} \cup \set{0}$). Formally, our \IFC problem is defined as follows.



\begin{definition}[\IFC (\textbf{IFC})]
    The input is a set $V$ of $n$ points with a $q$-length feature vector $X_v=(x_v^1,\ldots,x_v^q)$ for each $v\in V$, a similarity parameter $\gamma\in[0,1]$, an integer tuple $(m_v)_{v\in V}$,  a set $F$ of potential facilities. The objective is to open a subset $S\subseteq F$ of at most $k$ facilities, and find an assignment $\phi:V\longrightarrow S$ to minimize Cost$(C)$ satisfying the fairness constraints (eq., ~\ref{eq:p-ifc-constraint}). 
\end{definition}


The classical clustering problem, which we call \VC, is the same as the IFC problem except for the fairness requirements from Equation \ref{eq:p-ifc-constraint}.






%% file: sec_results.tex
\section{Results}

We present our main technical results in two directions. First, we provide several hardness results to show that the general \IFC (IFC) problem is hard even if one considers approximation. Then we contrast the hardness results by developing randomized approximation algorithms for various special cases of the IFC problem.

\subsection{Hardness Results}

In order to prove hardness results, we consider the decision version of the IFC problem, where the goal is to find a clustering whose cost is below a certain threshold. Note that, it is always possible to find a (trivial) individually fair clustering by one cluster containing all the points. However, the cost of such a fair clustering could be high, and we ask whether it is possible to beat the cost of such trivially fair clustering. As there can be multiple trivial fair clustering (depending on the cluster center chosen), we naturally pick the one minimizing the cost as the benchmark.

\begin{definition}[\TFC]
\label{defn:tfc_trivial}
    Given a set $V$ of $n$ points along with $q$-length feature vector $X_v = (x^1_v,\ldots,x^q_v)$ for every point $v \in V$, the trivially fair clustering puts all points in one cluster and picks the point as cluster center which minimizes the cost:
    $$
    \min_{f\in F} \sum_{v \in V} d(v, f)^p.
    $$
\end{definition}


We show that it is \NPC to compute if there exists a clustering better than \TFC by providing a reduction from \SPR, which is known to be \NPC~\cite{BTV06}.

\begin{definition}[\SPR]
    Given a graph $\GG=(\VV,\EE)$ and an integer $\lambda_v$ for every vertex $v\in \VV$, compute if there exists a partition $(\VV_1,\VV_2)$ of \VV such that
    \begin{enumerate}[i)]
        \item $\VV_1, \VV_2 \ne \emptyset$
        \item For every $i\in[2]$ and every $v\in\VV_i$, the number of neighbors of $v$ in $\VV_i$ is at least $\lambda_v$.
    \end{enumerate}
    We denote an arbitrary instance of \SPR by $(\GG,(\lambda_v)_{v\in \GG})$.
\end{definition}

\begin{theorem}
\label{thm:ifc_np_hard}
It is \NPC to decide whether an instance of \IFC admits a clustering of cost less than the \TFC even when there are only 2 facilities.
\end{theorem}
Please see the proof in the Supplementary.\\

Given the NP-completeness result, we explore the possibility of approximation for the \IFC (IFC) problem. However, the next theorem shows that IFC is inapproximable within factor $\delta$ for any $\delta>0$.

\begin{theorem}
\label{thm:noapprox}
Distinguishing between instances of the IFC problem having zero and non-zero optimal costs is NP-complete even when there are 2 facilities. Hence, for any computable function $\delta$, there does not exist a $\delta$-approximation algorithm for IFC\ unless P=NP.
\end{theorem}
\begin{proof}
Let $\AA$ be a deterministic polynomial time algorithm that distinguishes between instances of IFC with 0 and non-zero optimal costs. We use $\AA$ to build an algorithm for \SPR. Let $(\GG=(\VV=\{v_1,v_2,\ldots,v_n\},\EE), (\lambda_v)_{v\in\VV})$ be an instance of \SPR. We create $n-1$ instances $I_1,I_2,\ldots,I_{n-1}$ of IFC as follows: the set of points is $U=\{u_1,\ldots,u_n\}$ and $m_{v_i}=\lambda_{v_i}$ for every $i\in[n]$ for all the instances; for instance $I_i, i\in[n]$, we introduce $2$ facilities $l$ and $r$ and define distances as follows:
\[
  d(u_j,l) =
  \begin{cases}
                                    1 & \text{$j=1$} \\
                                   0 & \text{$j\in \{2,3,\ldots,n\}$} 
  \end{cases}
\]
\[
  d(u_j,r) =
  \begin{cases}
                                    1 & \text{$j=i+1$} \\
                                   0 & \text{$j\in \{1,2,\ldots,n\} \setminus \{i+1\}$} 
  \end{cases}
\]
We define feature vectors of every point similar to \Cref{thm:ifc_np_hard} to realize the above distances.
We now run $\AA$ on each of these instances. If $\AA$ finds any instance in $\{I_1, \ldots, I_{n-1}\}$ to have zero optimal cost, then we return yes for the \SPR instance; otherwise, we return no for the \SPR instance.

Clearly, the above algorithm runs in polynomial time. We now prove its correctness. If $\GG$ does not have a non-trivial satisfactory partition, then clearly every instance in $I_1, …, I_{n-1}$ has an optimal cost of one. On the other hand, if $\GG$ has a non-trivial satisfactory partition, say $(X,\Bar{X})$, then we claim that at least one of $I_1, …, I_{n-1}$ has optimal cost 0. Without loss of generality, we assume  that we have $v_1 \in X$. Let $v_j \in \Bar{X}$, for some $j \in \{2,3,\ldots,n\}$. Then clearly OPT$(I_{j-1})=0$ (assigning all the corresponding vertices in $X$ to the cluster center $r$ and all other vertices to the cluster center $l$). Thus, the algorithm is correct.
\end{proof}

Note that the distances in the above reduction do not satisfy the triangle inequality. If we insist that distances must satisfy the triangle inequality, then we have the following (weaker than \Cref{thm:noapprox}) inapproximability result.

\begin{theorem}
\label{thm:fptas_triangle}
There does not exist any FPTAS for the IFC problem when the distances in the input satisfy triangle inequality unless $\Pb=\NP$.
\end{theorem}
Please see the proof in the Supplementary.

\subsection{Algorithmic Results}
\label{sec::algo_results}


Given the strong inapproximability results in the previous section, we aim to develop approximation algorithms for \IFC under suitable conditions. First, we develop an approximation algorithm for \IFA  (\cref{thm:approx_algo_lp}). Next, we show how to obtain an algorithm for \IFC of similar guarantee (\cref{thm:final-algo}). Bera et al.~\cite{bera2019fair} designed an approximation algorithm for group fair clustering from an algorithm for group fair assignment. We follow a similar approach to individual fairness.


One of the main ingredients of our technical results is the \IFA problem, which, given a set of $k$ potential cluster centers, determines an assignment of the points i.e. which point should be assigned to which cluster center. Formally, it is defined as follows:
\begin{definition}[\IFR (\textbf{IFA})]
    Given a set $V$ of $n$ points along with a $q$-length feature vector $X_v=(x_v^1,\ldots,x_v^q)$ for every point $v\in V$, a similarity parameter $\gamma\in[0,1]$, an integer tuple $(m_v)_{v\in V}$, and a set $F=\{f_1,\ldots,f_k\}$ of $k$ facilities, an $(m_v)_{v\in V}$-fair assignment finds the optimal cost-minimizing assignment satisfying the fairness constraints ~\eqref{eq:p-ifc-constraint}. 
\end{definition}


\textbf{Our Algorithm, LP-FAIR: } Algorithm \ref{alg:ifa} describes our randomized approximation algorithm for IFA. The linear program (LP) in \cref{eq:lp-ifa} is a relaxation of the IFA problem. It has a variable $x_{v,f_k}$ for each vertex $v$ and facility $f_k$. In an (integral) ``solution" the variable $x_{v,f_k}$ takes value $1$ if and only if the point $v$ is assigned to the facility $f_k$. After solving the LP, Algorithm~\ref{alg:ifa} determines the assignment $\phi$ by assigning point $v$ to $f_k$ with probability $x^*_{v,f_k}$. Finally, the above procedure is repeated $\log_{1+\delta} n$ times and the assignment with the lowest cost is returned to boost the success probability.
\begin{algorithm}[!htbp]
 \caption{LP-FAIR, Algorithm for \textbf{IFA} \label{alg:ifa}}
  \begin{algorithmic}[1]
    \Require{$(V, (X_v)_{v\in V}, \gamma, (m_v)_{v\in V}, k)$, and $\delta$.}
    \For{$t = 1,2,\ldots,T = \log_{1+\delta} n$}
    \State{Solve the following LP to get solution $x^\star_t$.}
    \begin{align}\label{eq:lp-ifa}
    \begin{split}
    \min_{x} 
    &\sum\limits_{v\in V} \sum\limits_{f_k\in F} d(v,f_k)^p\cdot x_{v,f_k} \\
    \textrm{s.t.} &\sum\limits_{u\in \Gamma(v)} x_{u,f_k} \ge m_v \cdot x_{v,f_k} \ \forall v\in V, f_k\in F\\
    &\sum\limits_{f_k\in F} x_{v,f_k} = 1 \ \forall v\in V\\
    &x_{v,f_k} \ge 0 \ \forall v\in V, f_k\in F
    \end{split}
    \end{align}
    \For{each $v \in V$}
    
        \State{Set $\phi_t(v)= f_k$ with probability $x^\star_{t,v,f_k}$.}
    \EndFor
    \EndFor
\Return Assignment $\phi^\star$ with the minimum cost.
  \end{algorithmic}
\end{algorithm}


The next theorem presents probabilistic approximation guarantees provided by Algorithm~\ref{alg:ifa}.
\begin{theorem}
\label{thm:approx_algo_lp}
For any $\epsilon, \delta > 0$, 
there exists a randomized algorithm for IFA running in time polynomial in $n$ and $\frac{1}{\delta}$, that outputs a solution of cost at most $(1+\delta)$OPT where each vertex $v$ has at least $\frac{m_v}{k}(1-\epsilon)$ $\gamma$-similar points assigned to the same facility with high probability if $m_v = \Omega (\frac{k \log n}{\epsilon^2})$, $\forall v\in V$.
\end{theorem}
\begin{proof}
Let $x^\star$ be a solution to the linear program~\ref{eq:lp-ifa}. Algorithm \ref{alg:ifa} assigns point $v$ to $f_k$ with probability $x^*_{v,f_k}$. We now prove the quality of this solution. Let $X$ be the random variable denoting the number of points in $\Gamma(v)$ that are assigned to the same facility as a given point $v$. For every $u \in \Gamma(v)$, let $X_u$ be the indicator random variable indicating whether $u$ and $v$ are assigned to the same facility. Thus,

$$E[X_u] = \sum\limits_{f_k\in F} x^*_{v,f_k}x^*_{u,f_k}$$ 

So we have,
\begin{align*}
    E[X] &= \sum\limits_{u\in \Gamma(v)} E[X_u]\\
    &= \sum\limits_{f_k\in F} (x^*_{v,f_k}\sum\limits_{u\in \Gamma(v)} x^*_{u,f_k})\\
    &\ge m_v \sum\limits_{f_k\in F} x^{*^2}_{v,f_k} 
    \ge m_v/k
\end{align*}

Now using Chernoff bound, 
$$\text{Pr}[v \text{ has at most } \frac{m_v}{k}(1-\epsilon)\text{ $\gamma$-similar points}] \le e^{-\frac{\epsilon^2}{2}\frac{m_v}{k}} \le \frac{1}{n^2}$$ 

And using union bound,
$$\text{Pr}[\exists \text{ a vertex that has at most }\frac{m_v}{k}(1-\epsilon)\text{ $\gamma$-similar points}] \le \frac{1}{n}$$ 

Also, clearly, the expected cost of the computed solution is at most $\text{OPT}$. Hence, using Markov's inequality, 
$\text{Pr[cost of computed solution is } \ge (1+\delta)\text{OPT}] \le \frac{1}{1+\delta}$ 

As we repeat the above algorithm $T = \frac{\log n}{\log (1+\delta)}$ times and output the solution with minimum cost, we have
$$\text{Pr[cost of the computed solution is } \ge (1+\delta)\text{OPT}] \le \frac{1}{n}$$ 

Also, by union bound, the probability that there exist a vertex having at most $\frac{m_v}{k}(1-\epsilon)$ $\gamma$-similar points in one of the $T$ solutions is at most $\frac{T}{n}$.
\end{proof}

\begin{algorithm}[!htbp]
 \caption{Algorithm for \IFC\label{alg:ifc}}
  \begin{algorithmic}[1]
    \Require{$(V, (X_v)_{v\in V}, \gamma, (m_v)_{v\in V}, k)$}
    \State{Solve clustering problem $(V, (X_v)_{v\in V}, \gamma, (m_v)_{v\in V}, k)$ using any vanilla algorithm ignoring fairness constraints. Let $\left(\left(C_i\right)_{i\in[\el]},\left(c_i\right)_{i\in[\el]}\right)$ be the output of the clustering algorithm.}
    \State{$F\leftarrow\{c_1,\ldots,c_\el\}$}
    \State{Run algorithm for \IFA on $(V, (X_v)_{v\in V}, \gamma, (m_v)_{v\in V}, F)$. Let $\phi$ be the output.}
    \Return $(\phi^{-1}(c_i))_{i\in[\el]}$ (ignore $\phi^{-1}(c_j)$ if $\phi^{-1}(c_j)$ is the empty set for some $j\in[\el$)
  \end{algorithmic}
\end{algorithm}

 We next show a method to obtain an approximation algorithm for \IFC from an approximation algorithm for \IFA in a black box fashion.


\begin{theorem}\label{thm:final-algo}
If the distances satisfy the triangle inequality, then the existence of an $\rho$-approximation algorithm for \VC and an $\alpha$-approximation algorithm for \IFA with $\lambda$-multiplicative violations for some $\lambda\ge 1$ implies the existence of an $\alpha (\rho+2)$-approximation algorithm for \IFC with $\lambda$-multiplicative violation.
\end{theorem}
\begin{proof}
Let $S^*$ be the optimal set of facilities opened and $\phi^*$ be the optimal assignment in the input \IFC instance. Let $S$ be the set of facilities returned by the vanilla $k$-clustering problem and $\phi$ the assignment. For each $f^* \in S^*$, let us define $\text{nrst}(f^*)=\arg\min_{f\in S} d(f,f^*)$. Consider the assignment $\phi'$ over the set of facilities $S$ that assigns each vertex $v$ to $\text{nrst}(\phi^*(v))$. We claim that $\phi'$ is a fair assignment (please see Claim \ref{claim:phi_fair}).


For any vertex $v$, let $\phi(v)=f$, $\phi'(v)=f'$ and $\phi^*(v)=f^*$. Thus $d(v,f')\le d(v,f^*)+d(f^*,\text{nrst}(f^*)) \le d(v,f^*) + d(f^*,f) \le 2d(v,f^*) +d(v,f)$. Since $l_p$ is a monotone norm, $l_p(S,\phi')\le 2l_p(S^*,\phi^*)+l_p(S,\phi)\le (\rho+2)OPT$. Now, since we have an $\alpha$-approximation algorithm for \IFA, the solution returned by the algorithm will have cost at most $\alpha \cdot l_p(S,\phi') \le \alpha (\rho+2)\text{OPT}$.
\end{proof}
Note that the proof requires that $\phi'$ is an individually fair assignment. Next, we prove the following claim:

\begin{claim}
\label{claim:phi_fair}
$\phi'$ is an individually fair assignment.
\end{claim}

\begin{proof}
For $v\in V$, let $T_v$ denote the set of vertices assigned to $\phi^*(v)$. Since $\phi^*$ is an individually fair assignment, the number of $\gamma$-similar points of $v$ in $T_v$ is at least $m_v$. Now all vertices in $T_v$ are assigned to the facility $\text{nrst}(\phi^*(v))$. Thus, the number of $\gamma$-similar points of $v$ in the assignment $\phi'$ is at least $m_v$. Hence, $\phi'$ is an individually fair assignment. 
\end{proof}

%% file: sec_experiments.tex
\section{Experiments}
\label{sec:exp}

\begin{table*}[ht]
\centering
\setlength{\tabcolsep}{4pt}
\renewcommand{\arraystretch}{1}
\resizebox{\columnwidth}{!}{
\begin{tabular}{cccccccccc}
\toprule
& \multicolumn{3}{c}{\textbf{Normalized Cost}} & \multicolumn{3}{c}{\textbf{Fairness}} & \multicolumn{3}{c}{\textbf{Macro Fairness}} \\
\midrule
& \textsc{Adult}& \textsc{Bank} & \textsc{Diabetes} &
\textsc{Adult}& \textsc{Bank} & \textsc{Diabetes} &
\textsc{Adult}& \textsc{Bank} & \textsc{Diabetes} \\
\midrule
FairCenter & $0.544 \pm 0.107$ & $0.528 \pm 0.134$ & $0.341 \pm 0.029$ & $90.0 \pm 5.5$ & $96.0 \pm 2.0$ & \underline{$94.1 \pm 5.0$} & $64.0 \pm 15.0$ & \underline{$76.0 \pm 8.0$} & $64.0 \pm 8.0$ \\
Alg-PP & $0.625 \pm 0.098$ & $0.516 \pm 0.053$ & $0.422 \pm 0.089$ & $86.8 \pm 10.9$ & $91.3 \pm 7.4$ & $88.1 \pm 7.0$ & \underline{$70.0 \pm 24.5$} & $72.9 \pm 16.3$ & $60.0 \pm 20.0$ \\
Alg-AG & $0.617 \pm 0.114$ & $0.563 \pm 0.178$ & $0.649 \pm 0.268$ & $86.8 \pm 10.9$ & $83.4 \pm 5.5$ & $87.2 \pm 7.9$ & \underline{$70.0 \pm 24.5$} & $56.9 \pm 13.2$ & $60.0 \pm 20.0$ \\
P-PoF-Alg & $0.592 \pm 0.084$ & $0.586 \pm 0.136$ & $0.528 \pm 0.150$ & $86.8 \pm 10.9$ & $87.6 \pm 7.7$ & $92.8 \pm 9.1$ & \underline{$70.0 \pm 24.5$} & $66.9 \pm 18.0$ & \underline{$80.0 \pm 24.5$} \\
H-S & \underline{$0.267 \pm 0.053$} & \underline{$0.251 \pm 0.022$} & $0.107 \pm 0.033$ & \underline{$90.8 \pm 4.3$} & $\textbf{97.7} \pm \textbf{1.8}$ & $91.1 \pm 7.1$ & $53.0 \pm 13.3$ & $70.0 \pm 8.4$ & $66.0 \pm 21.5$ \\
Gonzalez & $0.331 \pm 0.084$ & $0.327 \pm 0.022$ & \underline{$0.088 \pm 0.029$} & $88.6 \pm 3.5$ & $90.4 \pm 3.9$ & $89.9 \pm 4.0$ & $56.0 \pm 15.0$ & $60.0 \pm 0.0$ & $56.0 \pm 15.0$ \\
LP-FAIR & $\textbf{0.194} \pm \textbf{0.034}$ & $\textbf{0.176} \pm \textbf{0.016}$ & $\textbf{0.057} \pm \textbf{0.026}$ & $\textbf{92.3} \pm \textbf{0.7}$ & \underline{$96.3 \pm 4.6$} & $\textbf{97.9} \pm \textbf{2.8}$ & $\textbf{80.0} \pm \textbf{0.0}$ & $\textbf{92.0} \pm \textbf{9.8}$ & $\textbf{92.0} \pm \textbf{9.8}$ \\
\bottomrule
\end{tabular}
}
\caption{Normalized cost and fairness comparison between LP-FAIR (ours) and competing baselines. The best and second-best values for each column are in bold and underlined, respectively. Our method outperforms or has performance comparable to the baselines in terms of the three evaluation metrics. LP-FAIR is able to increase the clustering fairness while keeping the distance between individuals smaller. For Normalized Cost, LP-FAIR is 34.5\% better than the best cost baseline, H-S, on average. For fairness, the best baselines are FairCenter and P-PoF-Alg, which are outperformed by our method by 12.5\% on average.}
 \label{tab::cost_fairness_methods}


 \end{table*}

In this section, we provide an experimental evaluation of our proposed LP-based algorithm along with six baselines using the cost function from K-means (unless specified otherwise) on three different datasets. Our evaluation consists of performance on different metrics (Sec. \ref{sec::performance}), the cluster quality (Sec. \ref{sec::quality_of_clusters}), and the effect of varying the number of clusters (Sec. \ref{sec::varying_number_clusters}). We also provide the running time (Supp. \ref{sec::running_time}) and additional experiments (Supp. \ref{sec::additional_exp}) in the Supplementary. 
Our implementation\footnote{\href{https://anonymous.4open.science/r/lp-fair}{https://anonymous.4open.science/r/lp-fair}} is available online anonymously.

\paragraph{Datasets:} We use three datasets from the UCI repository in our experiments.\footnote{\url{https://archive.ics.uci.edu/ml/datasets}} These datasets have also been used by previous work \cite{bera2019fair,MV20,CKLV17}.
We consider the following attributes for distance and $\gamma$-similarity (fairness):
\begin{itemize}
\item \textbf{\textsc{Adult} \cite{kohavi1996scaling}:} Cluster labels determine whether a person makes over 50K a year. Distance features are ``educationnum" and ``age". The $\gamma$-similarity features are ``salary" and ``hoursperweek".
\item \textbf{\textsc{Bank} \cite{moro2014data}: } Data from customers of a bank. Distance features are ``duration", and ``age", $\gamma$-similarity features are ``education" and ``balance".
    \item \textbf{\textsc{Diabetes}: } Data from diabetes patients from 130 hospitals in the USA from 1999 to 2008. The distance features are ``age" and ``number-emergency". The $\gamma$-similarity features are ``time\_in\_hospital", ``num\_lab\_procedures".
\end{itemize}
We also provide experiments with randomly selected features in the Supplementary.

\paragraph{Algorithms:} We evaluate the following seven algorithms. 
\begin{itemize}
    \item \textbf{Our LP-based approach (LP-FAIR): } Provides probabilistic approximation guarantees (Algorithm~\ref{alg:ifa}). Our implementation is based on Algorithm \ref{alg:ifc} (Section \ref{sec::algo_results}).
    \item \textbf{FairCenter \cite{JKL19}}: Ensures fairness based on the existence of a cluster center nearby. Notice that this algorithm has different fairness criteria than ours.
    \item \textbf{Alg-PP \cite{chakrabarti2022new}}: Optimizes the clusters based on a per-point fairness metric.
    \item \textbf{Alg-AG \cite{chakrabarti2022new}}: Optimizes the clusters based on an aggregate fairness metric.
    \item \textbf{P-PoF-Alg \cite{chakrabarti2022new}}: Incorporates the Price of Fairness notion and combines the constraints from Alg-PP and Alg-AG.
    \item \textbf{Hochbaum-Shmoys (H-S) \cite{hochbaum1985best}}: Uses the triangle inequality to solve a k-center problem with a 2-approximation algorithm. 
    \item \textbf{Gonzalez \cite{gonzalez1985clustering}}: Minimizes the maximum intercluster distance. H-S and Gonzalez are used as baselines in \cite{chakrabarti2022new} for comparison.
\end{itemize}

\paragraph{Performance measures:} We evaluate the algorithms described above using the following metrics:

\begin{itemize}
    \item \textbf{Normalized Cost:} Clustering cost (Equation \ref{eq:p-cost}) normalized by the cost of trivially fair clustering (Definition \ref{defn:tfc_trivial}). The normalization removes the effect of dataset-dependent feature distributions and makes it easier to compare the results across datasets: $\text{Normalized Cost}(A)= \frac{ Cost(A)}{Cost(\text{\TFC})}$.
    \item \textbf{Fairness}: This denotes the fraction of points that satisfy individual fairness. 
    \item \textbf{Macro Fairness}: This denotes the average of the Fairness metric for each cluster. 
    \item \textbf{Cluster Imbalance}: This measures the imbalance of the found clusters in terms of their sizes. It is a standard deviation of cluster sizes (i.e., the number of elements in the cluster). The lower value of imbalance means the clusters are more balanced.
\end{itemize}

\paragraph{Other settings:} We choose 200 points from each dataset randomly for all the experiments. The experiments are run five times, and averages and standard deviations are reported based on these repetitions. Let $\Gamma_v$ denote the number of $\gamma$-matched points of $v$ in the entire dataset and $k'$ be the number of clusters found by the algorithm. The initial $m_v$ for each node $v$ is set as $\frac{\theta}{k}*|\Gamma_v|$, then scaled by $\frac{k}{k'}$ (after $k'$ is decided) to make the fairness metric cluster balanced. Unless specified otherwise, we set $\theta$ as $0.5$ in all experiments.

\begin{figure*}[ht]
\vspace{-1mm}
\centering
\includegraphics[keepaspectratio, width=0.91\textwidth]{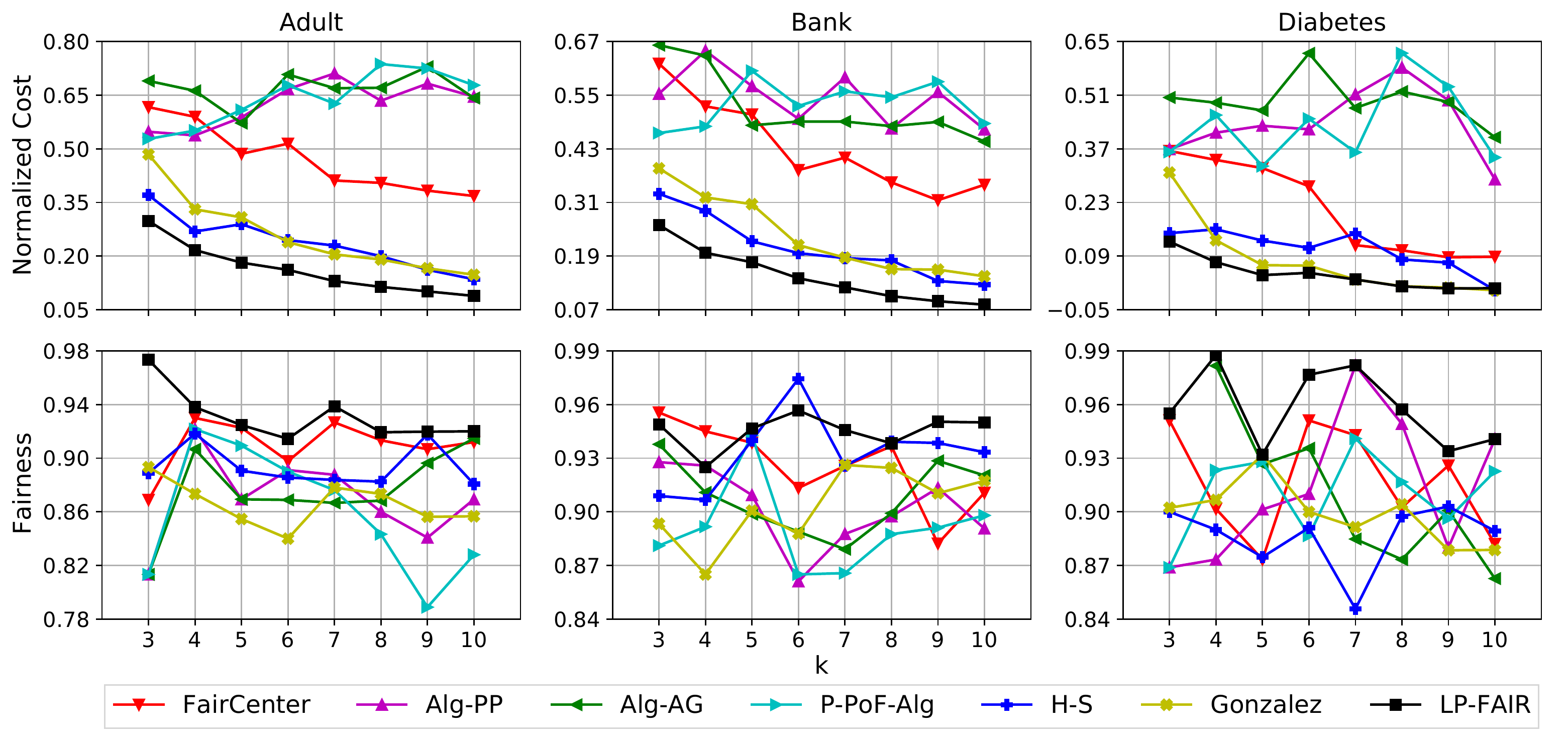}
\vspace{-1mm}
\caption{Cost and fairness results, varying the number of clusters ($k$) for $30k$ random data points using all datasets and methods (better seen in color). LP-FAIR outperforms or achieves results comparable to the baselines for all $k$. While increasing $k$ decreases the cost, it is less correlated with the fairness metric.}
\label{fig::different_k}
\vspace{-2mm}
\end{figure*}

\subsection{Performance \label{sec::performance}}

We present the results for our method LP-FAIR and competing baselines using the normalized cost (the lower the better) and fairness metrics (the higher the better). Table \ref{tab::cost_fairness_methods} shows the results produced by all algorithms. LP-FAIR has a significantly lower cost than the baselines, with a 34.5\% lower cost than the best baseline (H-S) on average. Moreover, LP-FAIR consistently clusters points fairer. The best baselines for the fairness metric are FairCenter and P-PoF-Alg, which have consistent performances overall. However, LP-FAIR outperforms them by 12.5\% on average. Furthermore, our method generates results with less variance compared to the baselines, which shows the stability of our algorithm.

\begin{table}[ht]
\centering
\begin{tabular}{cccc}
\toprule
& \textsc{Adult}& \textsc{Bank} & \textsc{Diabetes} \\
\midrule
FairCenter & $5.0 \pm 0.0$ & $5.0 \pm 0.0 $ & $4.0 \pm 0.63$ \\
Alg-PP & $2.0 \pm 0.0$ & $3.2 \pm 0.98$ & $2.0 \pm 0.0$ \\
Alg-AG & $2.0 \pm 0.0$ & $2.8 \pm 0.4$ & $1.8 \pm 0.4$ \\
P-PoF-Alg & $2.0 \pm 0.0$ & $4.0 \pm 1.67$ & $2.0 \pm 0.0$ \\
H-S & $4.6 \pm 0.49$ & $4.6 \pm 0.49$ & $4.6 \pm 0.49$ \\
Gonzalez & $5.0 \pm 0.0$ & $5.0 \pm 0.0$ & $5.0 \pm 0.0$ \\
LP-FAIR & $5.0 \pm 0.0$ & $5.0 \pm 0.0$ & $4.8 \pm 0.4$ \\
\bottomrule
\end{tabular}
\caption{The mean and standard deviation of the number of clusters generated by the methods (with $k=5$). Generating fewer clusters generally leads to higher costs (see Table \ref{tab::cost_fairness_methods}).}
 \label{tab::cluster_counts}
 \end{table}
 
\begin{table}[ht]
\centering
\begin{tabular}{cccc}
\toprule
& \textsc{Adult}& \textsc{Bank} & \textsc{Diabetes} \\
\midrule
FairCenter & $28.4 \pm 5.8$ & $30.3 \pm 5.2$ & $35.1 \pm 6.3$ \\
Alg-PP & $51.4 \pm 6.5$ & $37.1 \pm 13.6$ & $63.8 \pm 14.8$ \\
Alg-AG & $51.4 \pm 6.5$ & $42.7 \pm 13.8$ & $54.4 \pm 28.7$ \\
P-PoF-Alg & $51.4 \pm 6.5$ & $29.1 \pm 15.5$ & $37.6 \pm 23.0$ \\
H-S & $44.5 \pm 9.5$ & $34.9 \pm 6.2$ & $40.9 \pm 13.4$ \\
Gonzalez & $40.9 \pm 10.0$ & $33.3 \pm 5.7$ & $37.1 \pm 10.0$ \\
LP-FAIR & $17.9 \pm 2.8$ & $17.6 \pm 2.1$ & $15.8 \pm 4.8$ \\
\bottomrule
\end{tabular}

\caption{The mean and standard deviation of cluster imbalance. Imbalanced clusters result in small clusters where the members might not be individually fair. LP-FAIR generates clusters with lower imbalance compared to the baselines.}
 \label{tab::cluster_balance}
 \end{table}

\subsection{Quality of Clusters \label{sec::quality_of_clusters}}

Cluster quality is also a crucial metric to be considered. Table \ref{tab::cluster_counts} and Table \ref{tab::cluster_balance} show generated cluster counts and imbalance, respectively, for all methods and datasets. Some algorithms use less number of clusters even though the expected number is $k=5$, which makes their cost higher compared to the one generating close to $5$. Cluster imbalance is also critical as imbalanced clusters would result in clusters with a small number of elements in them. The elements in those clusters are unlikely to be fair, and this will affect individual fairness as well as macro fairness. LP-FAIR generates more balanced clusters (low imbalance) compared to the baselines resulting in better fairness.

\subsection{Varying Number of Clusters \label{sec::varying_number_clusters}}

Here, we evaluate the impact of the number of clusters ($k$) on cost and fairness. We vary the value of $k$ from 3 to 10 and choose $30k$ random points from each dataset. Figure \ref{fig::different_k} shows the results. LP-FAIR achieves the best results for most values of $k$. In general, the cost decreases as the number of clusters increases. The exceptions are for the cases where the number of clusters generated is lower than expected. For fairness metrics, LP-FAIR outperforms or achieves comparable results to all baselines for all datasets. The values for Fairness do not present a clear trend based on $k$. We provide results on Macro Fairness in Supplementary \ref{sec::macro_varying_cluster_count}.


 \paragraph{Summary.} A few key observations from the above experiments are as follows: (1) Our method (LP-FAIR) produces clusters that are individually fair to more than 95.5\% points and achieves 88\% Macro Fairness on average, outperforming the baselines in most of the settings; (2) LP-FAIR also produces lower cost or distance than competing baselines in all settings;
(3) LP-FAIR generates better clusters in terms of numbers and imbalance. This makes LP-FAIR clusters less costly and more fair compared to the existing competitors.

%% file: sec_conclusion.tex
\section{Conclusion}
We have studied the $k$-clustering problem with individual fairness constraints. Our notion of fairness is defined in terms of a feature-based similarity among points and guarantees that each point will have a pre-defined number of similar points in their cluster. We have provided an algorithm with probabilistic approximation guarantees for optimizing the cluster distance as well as ensuring fairness. Finally, the experimental results have shown that our proposed algorithm can produce $34.5\%$ lower clustering cost and $12.5\%$ higher individual fairness than previous works on average. 


%% file: appendix.tex
\section{Proof of Theorem \ref{thm:ifc_np_hard}}

\begin{proof}
The problem clearly belongs to \NP. We now exhibit a reduction from \SPR. Let $(\GG=(\VV=\{v_1,v_2,\ldots,v_n\},\EE),(\lambda_v)_{v\in\VV})$ be an arbitrary instance of \SPR. In our \IFC instance, we have two facilities $l$ and $r$ and the set of points  $U=\{u_1,u_2,\ldots,u_n\}$. We define the distances as
\[
  d(u_i,l) =
  \begin{cases}
                                   (\lceil \frac{n}{2} \rceil + \beta)^{1/p} & \text{$1 \le i \le \lceil \frac{n}{2} \rceil + 1$} \\
                                   \beta^{1/p} & \text{$\lceil \frac{n}{2} \rceil + 2 \le i \le n$} 
  \end{cases}
\] 
\[
  d(u_i,r) =
  \begin{cases}
                                    \beta^{1/p} & \text{$1 \le i \le \lfloor \frac{n}{2} \rfloor$} \\
                                   (\lceil \frac{n}{2} \rceil + \beta + 1)^{1/p} & \text{$\lfloor \frac{n}{2} \rfloor + 1 \le i \le n$} 
  \end{cases}
\]
for any $\beta \ge \frac{\lceil \frac{n}{2}\rceil + 1}{2}$. We can easily verify that the following properties are satisfied:
\begin{enumerate}[i)]
    \item The distances satisfy the triangle inequality.\\
    \item $\sum\limits_{v\in V} d(v,l)^p = \sum\limits_{v\in V} d(v,r)^p = A$ (say).\\
    \item For any $ X \subset V, X\ne \emptyset, X\ne V,$ we have $\sum\limits_{v\in X} d(v,l)^p \neq \sum\limits_{v\in X} d(v,r)^p$. To see this, let $s = |X\cap \{u_1,u_2,\ldots,u_{\lceil \frac{n}{2} \rceil + 1}\}|$ and $t = |X\cap \{u_{\lfloor \frac{n}{2} \rfloor + 1},u_{\lfloor \frac{n}{2} \rfloor + 2},\ldots,u_n\}|$ so that $s \le \lceil \frac{n}{2} \rceil +1$ and $t \le \lceil \frac{n}{2} \rceil$. Then $\sum\limits_{v\in X} d(v,l)^p = s\lceil \frac{n}{2} \rceil + \beta |X|$ and $\sum\limits_{v\in X} d(v,r)^p = t(\lceil \frac{n}{2} \rceil+ 1) + \beta |X|$. Thus $\sum\limits_{v\in X} d(v,l)^p = \sum\limits_{v\in X} d(v,r)^p$ would imply $s\lceil \frac{n}{2} \rceil = t(\lceil \frac{n}{2}\rceil + 1)$ and therefore either $s=t=0$ or $s = \lceil \frac{n}{2} \rceil +1$ and $t = \lceil \frac{n}{2} \rceil$. In the former case $X=\emptyset$ while in the latter case $X=V$, a contradiction.\\
\end{enumerate}

We now describe the feature vector. For every edge $e\in \EE$, we have a feature $\theta_e$. A point $u_i, i\in[n]$ has value $1$ for $\theta_e$ if the edge $e$ is incident on the vertex $v_i$ in \GG; otherwise has value $0$ for $\theta_e$. We define the distance function $d^\prime$ on feature space as the number of features that two points differ. Finally, we set the similarity parameter $\gamma=\frac{e^{-m}+e^{-(m-1)}}{2}$. We observe that two points $u_i, u_j, i,j\in[n]$ are $\gamma$-similar if and only if there is an edge between $v_i$ and $v_j$ in \GG. Finally, we define $m_{v_i}=\lambda_{v_i}$ for every $i\in[n]$.

We claim that the \SPR instance is a yes instance if and only if there exists a fair clustering of $U$ with cost $< A$. The ``if" part follows directly, since any fair clustering of $U$ with cost $<A$ must be non-trivial and fairness ensures that the corresponding partition of \GG is satisfactory.

For the ``only if" part, let $(X,\Bar{X})$ be a non-trivial satisfactory partition of $\GG$. Let $\phi_1$ denote the assignment that assigns all corresponding vertices in $X$ to $l$ and all corresponding vertices in $\Bar{X}$ to $r$, and $\phi_2$ denote the assignment that assigns all vertices in $X$ to $r$ and all vertices in $\Bar{X}$ to $l$. Thus, \[ \cost(\phi_1) = \sum\limits_{v\in X} d(v,l)^p + \sum\limits_{v\in \Bar{X}} d(v,r)^p \] 
\[ \cost(\phi_2) = \sum\limits_{v\in X} d(v,r)^p + \sum\limits_{v\in \Bar{X}} d(v,l)^p \]
Thus, $\cost(\phi_1) + \cost(\phi_2) = 2A$. Now, it cannot be the case that $\cost(\phi_1)=\cost(\phi_2) = A$, which would  imply that one of the assignments $\phi_1$ or $\phi_2$ must have cost $<A$. Suppose to the contrary that $\cost(\phi_1)=A$. Thus, $\sum\limits_{v\in X} d(v,l)^p + \sum\limits_{v\in \Bar{X}} d(v,r)^p = \sum\limits_{v\in X} d(v,r)^p + \sum\limits_{v\in \Bar{X}} d(v,r)^p$ and therefore $\sum\limits_{v\in X} d(v,l)^p = \sum\limits_{v\in X} d(v,r)^p$, a contradiction.
\end{proof}

\section{Proof of Theorem \ref{thm:fptas_triangle}}
\begin{proof}
Let $\AA$ be an FPTAS for IFC. Similar to the proof of \cref{thm:noapprox}, we create $n-1$ instances of \SPR where instance $I_i$ is as follows: the set of points is $U=\{u_1,\ldots,u_n\}$ and $m_{v_i}=\lambda_{v_i}$ for every $i\in[n]$ for all the instances; for instance $I_i, i\in[n]$, we introduce $2$ facilities $l$ and $r$ and define distances as follows:
\[
  d(u_j,l) =
  \begin{cases}
                                    (1+\beta)^{1/p} & \text{$j=1$} \\
                                   \beta^{1/p} & \text{$j\in \{2,3,\ldots,n\}$} 
  \end{cases}
\]
\[
  d(u_j,r) =
  \begin{cases}
                                    (1+\beta)^{1/p} & \text{$j=i+1$} \\
                                   \beta^{1/p} & \text{$j\in \{1,2,\ldots,n\} \setminus \{i+1\}$} 
  \end{cases}
\]
where $\beta$ is any constant $\ge 1/2$. The algorithm runs $\BB$ on each of the above instances with approximation parameter $\epsilon = \frac{1}{2n\beta}$. If $\BB$ returns a solution of cost less than $1+n\beta$ on any instance, return yes for the \SPR instance; otherwise, we return no for the \SPR instance.

Clearly, the cost of the trivial partition is $1+n\beta$. Thus, if $\GG$ does not have a non-trivial satisfactory partition, then $\BB$ must always return the trivial assignment of cost $1+n\beta$ for all instances. If $\GG$ has a satisfactory partition $(X,\Bar{X})$, then as in \cref{thm:noapprox}, there exists an instance with optimal cost $n\beta$. Thus, the solution returned by $\BB$ will have cost at most $n\beta(1+\frac{1}{2n\beta}) < 1 + n\beta$. Hence, the algorithm is correct.
\end{proof}

\section{Running Time \label{sec::running_time}}

The most expensive part of our algorithm is the linear programming part where we use scipy.optimize.linprog Python module. Since our algorithm is randomized, we run 10 trials, but the best performance is achieved in at most 4 trials. One trial of our algorithm takes 36.7 seconds to finish on the Adult dataset. Table \ref{tab::running_times} shows the running time of LP-FAIR and the baselines on the Adult dataset.

\begin{table}[ht]
\footnotesize
\centering
\begin{tabular}{cc}
\toprule
& \textsc{Adult} \\
\midrule
FairCenter & $16.5$s \\
Alg-PP & $11.8$s \\
Alg-AG & $9.68$s \\
P-PoF-Alg & $9.55$s \\
H-S & $7.92$s \\
Gonzalez & $7.91$s \\
LP-FAIR (10 trials) & $367.5$s \\
LP-FAIR (4 trials) & $147$s \\
LP-FAIR (1 trial) & $36.7$s \\
\bottomrule
\end{tabular}

\caption{The running times of LP-FAIR and baselines on the Adult dataset. We run our randomized algorithm 10 times, but the best performance is achieved with at most 4 trials.}
 \label{tab::running_times}
 \end{table}

\begin{figure*}
\centering
\includegraphics[keepaspectratio, width=0.91\textwidth]{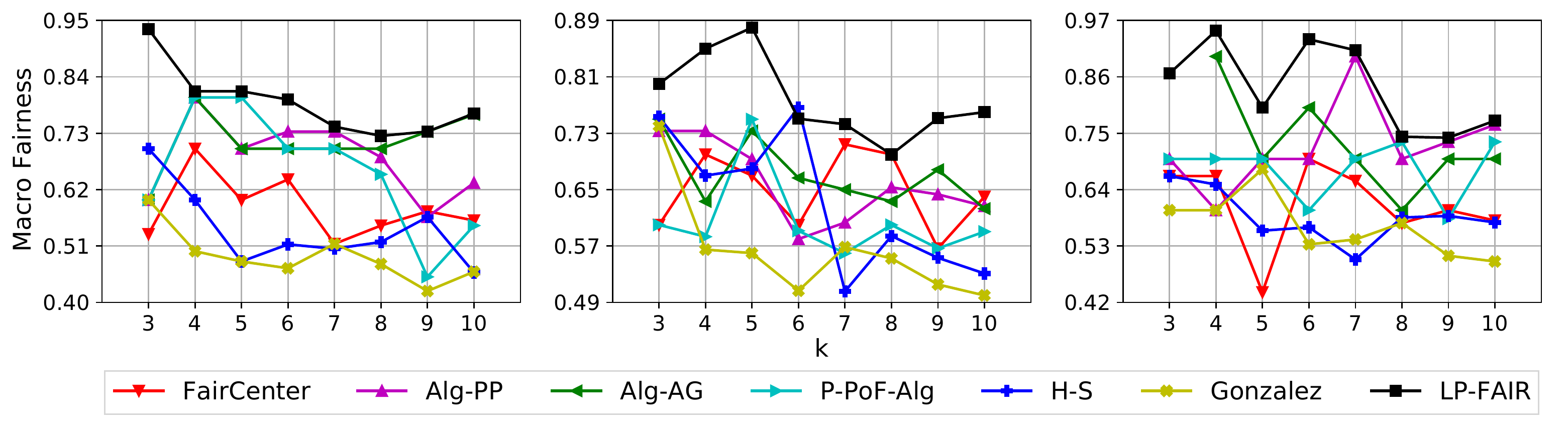}
\caption{Macro Fairness results, varying the number of clusters ($k$) for $30k$ random data points using all datasets and methods (better seen in color). LP-FAIR outperforms or achieves results comparable to the baselines for all $k$. Increasing $k$ decreases the macro fairness.}
\label{fig::different_k_macro_only}
\end{figure*}

\section{Additional Experiments}
\label{sec::additional_exp}

\subsection{Random Features \label{sec::random_features}}

We randomly select two distance and two fairness attributes, with five different random selections, and run each experiment five times. Table \ref{tab::cost_fairness_methods_random} shows that selecting random features does not alter the performance of our method compared to Table \ref{tab::cost_fairness_methods}, and LP-FAIR is the best performing or has comparable results to the baselines.

\begin{table}[ht]
\centering
\setlength{\tabcolsep}{4pt}
\resizebox{0.7\columnwidth}{!}{
\begin{tabular}{cccccccccc}
\toprule
& \multicolumn{3}{c}{\textbf{Normalized Cost}} & \multicolumn{3}{c}{\textbf{Fairness}} \\
\midrule
& \textsc{Adult}& \textsc{Bank} & \textsc{Diabetes} &
\textsc{Adult}& \textsc{Bank} & \textsc{Diabetes} \\
\midrule
FairCenter & $0.783$ & $0.535$ & $0.874$ & $91.93$ & $90.90$ & $92.43$ \\
Alg-PP & $0.632$ & $0.415$ & $0.852$ & $76.60$ & $77.38$ & $91.00$ \\
Alg-AG & $0.643$ & $0.432$ & $0.953$ & $78.55$ & $77.75$ & $91.30$ \\
P-PoF-Alg & $0.525$ & $0.382$ & $0.856$ & $76.60$ & $75.15$ & $94.20$ \\
H-S & \underline{$0.415$} & $0.526$ & $0.773$ & \underline{$92.55$} & $\textbf{93.60}$ & \underline{$94.90$} \\
Gonzalez & $0.472$ & \underline{$0.359$} & $0.548$ & $88.50$ & $91.20$ & $90.50$ \\
LP-FAIR & {\bf 0.302} & $\textbf{0.232}$ & $\textbf{0.189}$ & $\textbf{93.48}$ & \underline{$92.18$} & $\textbf{96.43}$ \\
\bottomrule
\end{tabular}
}
\caption{Normalized cost and fairness comparison between LP-FAIR (ours) and competing baselines with random feature selections. The best and second-best values for each column are in bold and underlined, respectively. Our method outperforms or has performance comparable to the baselines in terms of fairness and cost.}
 \label{tab::cost_fairness_methods_random}
 \end{table}

\subsection{Neglecting Fairness Constraint \label{sec::neglect_fairness_constraint}}

We run experiments by setting $m_v = 0$ in Equation \ref{eq:p-ifc-constraint} to see the effect of our fairness constraint. Setting $m_v = 0$ essentially makes all points fair after the first cluster assignments, so the cost will be minimized with K-means. Table \ref{tab::fairness_effect} shows that setting $m_v = 0$, as expected, decreases the normalized cost, whereas the fairness performance becomes much worse.

\begin{table}[ht]
\centering
\setlength{\tabcolsep}{4pt}
\resizebox{0.7\columnwidth}{!}{
\begin{tabular}{cccccccccc}
\toprule
& \multicolumn{3}{c}{\textbf{Normalized Cost}} & \multicolumn{3}{c}{\textbf{Fairness}} \\
\midrule
& \textsc{Adult}& \textsc{Bank} & \textsc{Diabetes} &
\textsc{Adult}& \textsc{Bank} & \textsc{Diabetes} \\
\midrule
LP-FAIR ($m_v = 0$) & $\textbf{0.182}$ & $\textbf{0.159}$ & $\textbf{0.047}$ & ${89.3}$ & {$92.2$} & ${93.4}$ \\
LP-FAIR & $0.194$ & $0.176$ & $0.057$ & $\textbf{92.3}$ & $\textbf{96.3}$ & $\textbf{97.9}$ \\
\bottomrule
\end{tabular}
}
\caption{The effect of removing fairness constraint from LP-FAIR. The better performances are in bold. As expected, our algorithm makes the clusters more costly while having more individually fair clusters.}
 \label{tab::fairness_effect}
 \end{table}

\subsection{Macro Fairness on Varying Number of Clusters \label{sec::macro_varying_cluster_count}}

We also evaluate the impact of the number of clusters on the macro fairness metric. Figure \ref{fig::different_k_macro_only} shows that LP-FAIR performs better or has comparable results to the baselines. Macro Fairness tends to decrease as the number of clusters increases. That is because $k$ affects the cluster imbalance. Eventually, this makes some clusters less fair and decreases the average score.

\section{Reproducibility}

The code and datasets are available anonymously here: \href{https://anonymous.4open.science/r/lp-fair}{https://anonymous.4open.science/r/lp-fair}

